\documentclass[runningheads]{llncs}

\usepackage[T1]{fontenc}
\usepackage{graphicx}
\usepackage[english]{babel}
\usepackage{amssymb}
\usepackage{amsmath}
\usepackage{hyperref}

\usepackage{adjustbox}
\usepackage{tikz-cd}
\usepackage{mathtools}
\usepackage{microtype}
\usepackage{listings}
\usepackage{algpseudocode}

\usepackage{stmaryrd}
\usepackage{latexsym}

 \DeclareMathOperator{\effect}{effect} \DeclareMathOperator{\causes}{causes}
 
 \DeclareMathOperator{\condition}{condition}

\DeclareMathOperator{\Graph}{Graph}

\begin{document}

\title{From probability to causality in \\ probabilistic logic programming}
%
\titlerunning{From probability to causality}
\author{Zora Wurm \inst{1} \and Kilian R\"{u}ckschlo\ss \inst{2} \and 
Felix Weitk\"{a}mper\inst{3,1}}
\authorrunning{Z. Wurm et al.}
%
\institute{ Ludwig-Maximilians-Universit\"at M\"unchen, Oettingenstraße 67, 80538 M\"unchen \and
Eberhard-Karls-Universität Tübingen, Auf der Morgenstelle 10, 72076 Tübingen
\email{kilian.rueckschloss@uni-tuebingen.de} 
 \and
German University of Digital Science, Marlene-Dietrich-Allee 14, 14482 Potsdam 
 \email{felix.weitkaemper@german-uds.de}
}
\maketitle

\begin{abstract}
  Probabilistic logic programming is a formalism of statistical relational artificial intelligence that supports causal queries, including interventions from outside the system. When the structure of a probabilistic logic program is learned from data, however, only probabilistic information is used, and a single probability distribution may be compatible with several causal orders.
  This leads to ambiguity in interventional reasoning, raising the question of when the causal order is uniquely determined by the distribution.
Exploiting the relationship between acyclic probabilistic logic programs and Bayesian networks, we derive conditions under which the probabilistic information encoded in a program determines a unique causal order. We also incorporate constraints arising from relational structure by taking into account prescribed sets of causal symmetries induced by the underlying relational vocabulary. The result is a method for verifying when a learned probabilistic logic program supports well-defined intervention semantics.
\end{abstract}


\section{Introduction}
\label{sec: Introduction}

Understanding the causes of events enables us to predict the consequences of actions that interfere with the internal mechanisms of a system. 
Whether building autonomous agents, researching medical treatments, or modelling the future developments of climate change -- understanding cause-effect mechanisms is crucial,  since statistical correlation alone is insufficient to assess the effects of outside intervention \cite{Causality}. 
Furthermore, many causal mechanisms in real-life systems are uncertain. 
Integrating probability into causal modelling allows us to quantify likelihood and randomness in the effects of our actions. 

Causal Bayesian networks represent causal relationships of events as a directed acyclic graph, where each node is annotated with a probability conditioned on the values of its direct causes. 
Since they capture causal dependencies in probabilistic settings, causal Bayesian networks also support the simulation of external interventions on the system they model. 
Probabilistic logic programming combines logical reasoning with probabilistic inference by labelling logical rules with probabilities \cite{riguzzi2022foundations}. 
Beyond probabilistic reasoning, probabilistic logic programs can also model interventions and even compute the probabilities of counterfactual events.   
The probability distributions induced by acyclic probabilistic logic programs can be described by an associated Bayesian network, and prior work has shown that the intervention notion in probabilistic logic programs is compatible with that of the associated Bayesian network \cite{ruckschloss2021exploiting}. 

In recent years, significant progress has been made in inducing both the structure and the probability parameters of a probabilistic logic program from data \cite{slipcover,liftcover}. 
This raises the question: can programs learned from probabilistic data answer interventional queries?
The old adage that ``correlation is not causation'' reminds us that in general, probabilistic dependency alone cannot determine the direction of cause and effect. 
Yet under suitable assumptions, constraints from probabilistic dependencies and independencies often do suffice to pinpoint the causal structure underlying a given probability distribution.   
Such investigations have a long history in the context of Bayesian networks \cite{pcalgo}, where Meek \cite{CausalInferenceAndCausalExplanationWithBackgroundKnowledge} developed a complete and correct method for determining when causal reasoning is admissible based solely on the probabilistic independencies encoded in the graph. In this contribution, we transfer these ideas to the setting of probabilistic logic programs. 

As a paradigm for statistical relational learning, probabilistic logic programs are typically valued for their ability to express relational knowledge and integrate first-order reasoning.
By abstracting program structure from a specific data instance, we can also hope to infer more general causal mechanisms. 
For example, in a relational data set describing various burning objects, we may detect a causal link between fire and smoke -- whenever something burns, it produces smoke, regardless of whether it is a bonfire, a house, or a cigarette.

By viewing the predicates of a relational vocabulary as expressing such symmetries, we can constrain the space of plausible causal explanations.
Assuming the symmetries encoded by the relational alphabet, this remains true even when we examine only a single ground instance.
To formalise this, we introduce a general notion of orientability relative to a set of causal symmetries, along with new orientation rules that exploit this additional constraint.

\section{Causal Bayesian Networks}
\label{subsec:CBN}

Pearl \cite{Causality} established that questions about the effect of external interventions on a probabilistic system cannot be answered when presented only with the probability distribution encoded by the system.
The simplest example is an isolated relationship between two random variables. Indeed, if $X$ and $Y$ are Boolean random variables, their joint distribution can only tell use the correlation between them;
it cannot distinguish between an effect of $X$ on $Y$, an effect of $Y$ on $X$ or indeed a hidden confounding variable $Z$ affecting both $X$ and $Y$.
For instance, if we merely observe that patients taking a particular drug are more likely to have high blood pressure, we cannot tell whether the drug increases blood pressure, high blood pressure causes people to take the drug, or in fact an underlying illness is responsible both for the patients' taking the drug and for the high blood pressure. Thus, we have no way of knowing the effect of prescribing the drug on a patient's blood pressure. 

However, if we not only have the distribution but also an understanding of the direction of causation, then we can predict the effect of external interventions.
To this end, Pearl \cite{Causality} encodes probability distributions and causal influences in causal Bayesian networks.
As we are ultimately interested in probabilistic logic programs, we restrict our attention to Boolean random variables.  

\begin{definition}[Causal Bayesian networks]
A (causal) Bayesian network $B$ consists of a directed acyclic graph $G$ on a set of Boolean random variables $V$ and, for every random variable $A \in V$ and every subset $T \subseteq \mathrm{Pa}(A)$ of the parents of $A$, a conditional probability $\mu_T(A) \in [0,1]$.   
\end{definition}

Syntactically, causal Bayesian networks are just ordinary Bayesian networks, and therefore induce a probability distribution in the usual way.
Note in particular that value assignments $v:V \rightarrow \{\top, \bot \}$, (possible configurations) of the random variables in $V$ can be seen as subsets of $V$, where $A \in v$ if and only id $v(A) = \top$.   

\begin{definition}[Probabilistic Bayesian network semantics]
The Bayesian network $B$ assigns to a value assignment  $v \subseteq V$ the probability
\[
\pi_B(v) := \pi(v) :=
\prod_{A \in V} \mu_{v \cap \mathrm{Pa}(A)}(A).
\]

\end{definition}

Note that according to this semantics, $\mu_T(A)$ is indeed the conditional probability that $A = \top$ given that for all $X \in \mathrm{Pa}(A)$,  $X = \top$ if and only if $X \in T$. 
This also implies that once the underlying directed acyclic graph is fixed, the parameters of the Bayesian network are unique. 

However, they are given the moniker \emph{causal} to emphasise that their arrows are oriented along the true causal ordering. Thus, they also support the computation of external interventions.

\begin{definition}[Intervention on a Bayesian network]
  Let $\mathcal{B}$ be a Bayesian network, $A \in V$ and $t \in \{\top, \bot\}$. Then the \emph{intervention} $\mathcal{B}_{\mathrm{do}(A = t)}$ is a Bayesian network on the graph $G_{\mathrm{do}(A = t)}$, where 
  \begin{enumerate}
      \item The graph $G_{\mathrm{do}(A = t)}$ has nodes $V$ and  edges $V\setminus \{(X,A) \mid X \in V\}$;
      \item The conditional probability $\mu_T(B)$ is the corresponding conditional probability in $\mathcal{B}$ if $B \in V \setminus A$;
      \item  The conditional probability $\mu_\emptyset(A) := 1$ if $t = \top$ and $\mu_\emptyset(A) := 0$ if $t = \bot$. 
  \end{enumerate}
\end{definition}

In summary, external interventions can be computed from a causal Bayesian network as long as the arrows define the true directions of causal flow.
In practice, the causal direction could be given as background knowledge, specified by a domain expert.
Frequently, however, the causal structure is not previously known, and then the structure of the Bayesian network is itself learned from data.
In that case, even assuming perfect learning, the structure is not guaranteed to match the causal relationships, as the learner cannot possible distinguish  multiple Bayesian networks with different graphs inducing exactly the same probability distribution on the basis of data sampled from that distribution.
On the other hand, if every possible Bayesian networks that encodes the same probability distribution as a given (learned) Bayesian network $\mathcal{B}$ has a directed edge from $A$ to $B$, then the existence of this edge does follow from the probabilistic data encoded by $\mathcal{B}$.

The directed acyclic graph underlying the Bayesian network determines probabilistic independencies that hold in any Bayesian network defined on it. 

\begin{definition}
    Let $G$ be a directed acyclic graph, let $Z$ be a set of nodes in $G$ and let $P = (A_1, \dots, A_n)$ be an undirected path in $G$. 
    Then $Z$ is said to \emph{block} $P$ if there is a $z \in Z$ and nodes $A$ and $B$ such that the directed edges $(A,z)$ and $(z,B)$ are on the path or such that $(z,A)$ and $(z,B)$ is on the path, 
    or if there is a $y$ such that $(A,y)$ and $(B,y)$ are on the path, and neither $y$ nor any descendant of $y$ lies in $Z$. 
    If $X$ and $Y$ be nodes in $G$ and $Z$ a set of nodes of $G$, then  
    $Z$ \emph{d-separates} $X$ and $Y$ if every undirected path from $X$ to $Y$ is blocked by $Z$.  

    If $G$ and $G'$ are directed acyclic graphs with the same set of nodes, they are \emph{Markov-equivalent} if for $X,Y,Z$ as above, $Z$ d-separates $X$ and $Y$ in $G$ if and only if $Z$ d-separates $X$ and $Y$ in $G'$.  
\end{definition}

This motivates the strategy adopted by Spirtes et al.~\cite{CausationPredictionAndSearch} and Meek \cite{CausalInferenceAndCausalExplanationWithBackgroundKnowledge}, who operate under the fundamental assumption that all independencies of a probability distribution are captured by the structure of the Bayesian network:

\begin{definition}[Faithfulness]
A probability distribution $\mu$ on a set of random variables $V$ is \emph{faithfully Markov} to a directed acyclic Graph $G$ on $V$ if for any two random variables $X,Y \in V$ and any subset $Z \subset ZV$, $X$ and $Y$ are probabilistically independent if and only if $Z$ d-separates $X$ and $Y$ in $G$.  
\end{definition}

\begin{remark}
The faithfulness condition has been shown to be satisfied Lebesgue almost always in a Boolean Bayesian network, in the sense that if the lists of numbers $\mu_T$ defining a Bayesian network are viewed as vectors in Euclidean space, the set of Bayesian network specifications violating the faithfulness condition has Lebesgue measure zero \cite{FaithfulnessHoldsAlmostAlways}.

However, this does not mean that faithfulness is not breached in practice. 
A common scenario are deterministic variables in Bayesian networks, all of whose associated probabilities are 0 or 1. Such variables frequently lead to unfaithfulness, as in the constellation $A \rightarrow B \rightarrow C$. If $B$ is deterministic and depending only on $A$, then conditioning on $A$ renders $B$ a constant and thus independent of its direct successor $C$. 
\label{rem:faithfulness}
\end{remark}
It turns out that two nodes $X$ and $Y$ are adjacent in a directed acyclic graph if and only if they are not d-separated by any $Z \subseteq V\setminus \{X,Y\}$. Therefore, any two Markov-equivalent graphs share the same adjacencies. 
This leaves only the direction of the edges as potentially variable. 

\begin{definition}[Orientability]
  A directed edge $(X,Y)$ in a directed acyclic graph $G$ is \emph{orientable} if it is contained in any graph that is Markov-equivalent to $G$. $G$ itself is called \emph{orientable} is every edge in $G$ is orientable.  
\end{definition}

Meek \cite{CausalInferenceAndCausalExplanationWithBackgroundKnowledge} showed that the set of all orientable edges can be determined through iterating a small number of local rules.

\begin{proposition}[Characterising orientability]
  The set of all orientable edges in a directed acyclic graph $G$ can be determined recursively as follows: 
  \begin{enumerate}
      \item If $(A,B)$ and $(C,B)$ be edges of $G$ such that $A$ and $C$ are not adjacent in $G$, then $(A,B)$ and $(C,B)$ are orientable. 
      \item If $(A,B)$ is orientable, $(B,C)$ an edge, and $A$ and $C$ not adjacent, then $(B,C)$ is orientable. 
      \item If $(A,B)$ and $(B,C)$ are orientable and $(A,C)$ is an edge, then $(A,C)$ is also orientable. 
      \item If $(B,D)$ and $(C,D)$ are orientable, $A$ adjacent to both $B$ and $C$ but $B$ not adjacent to $C$ and $(A,D)$ an edge, then $(A,D)$ is orientable. 
      \item If $(A,C)$ and $(C,D)$ are orientable, $B$ is adjacent to both $A$ and $C$ and $(B,D)$ is an edge, then $(B,D)$ is orientable.   
  \end{enumerate}
  \label{Prop:Meekrules}
\end{proposition}

\section{Propositional probabilistic logic programs}

We first introduce ground probabilistic logic programs and associate them with Bayesian networks. 

Let us fix a propositional alphabet ${P}$ in two parts with an \textbf{external vocabulary}~$E \subset P$ and an \textbf{internal vocabulary} $I := P \setminus E$. Here, $P$ is a finite propositional vocabulary, i.e.~it consists of a finite set of propositions. 
Propositions are also called \emph{atoms}, and atoms $p$ and their negations $\neg p$ are known as \emph{literals}. 

As we will be using encoded probabilistic independencies to identify possible network structures, we must restrict ourselves our random variables to those that are actually random \cite{independencies}, deviating from the original definition \cite{ProbLog}. 
\begin{definition}[ProbLog Clause]
A \emph{propositional ProbLog clause $RC$} is an expression of the form 
$\left(\pi :: R \leftarrow R_1,...,R_m, L_1,...,L_n. \right)$, written as 
$$\left(
\pi_{RC} :: \effect(RC) \leftarrow \causes(RC) \cup \condition (RC)
\right),
$$
which is given by the following data:
\begin{enumerate}
\item[i)]
an internal atom $R := \effect (RC)$, called the \emph{effect} of $RC$
\item[ii)]
a finite and possibly empty set of internal literals $\causes (RC) := \{ R_1,...,R_m \}$, called the \emph{causes} of~$RC$
\item[iii)]
a finite and possibly empty set of external literals $\condition (RC) := \{ L_1 , ... , L_n \}$, called the \emph{condition} of~$RC$  
\item[iv)]
a \emph{probability} $\pi (RC) \in (0,1)$ (note that we disallow $\pi (RC) = 1$). 
\end{enumerate}
We call $RC$ \emph{positive} if the set of causes $\causes(RC)$ contains only positive literals. 
\end{definition}

Propositional probabilistic logic programs define a distribution over valuations of the propositions in the internal vocabulary. 
To relate them to Bayesian networks, which are necessarily acyclic, we restrict ourselves to acyclic ProbLog programs. 

\begin{definition}[Propositional ProbLog program]
A \emph{propositional Problog program} $\Pi$ is a finite set of ProbLog clauses.
We say that $\Pi$ is \emph{acyclic} if, disregarding its probabilities, $\Pi$ is an acyclic logic program. 
\end{definition}

The probability distribution induced by such a program can be stored in a Bayesian network on its dependency graph, which depends on the active clauses. These are dictated by an external valuation $V$. 

\begin{definition}
Let $V$ be a valuation of the propositions in $E$. A clause of $\Pi$ is then said to be \emph{activated} by $V$ if its condition is satisfied by $V$. Then the relative program $\Pi^V$ with respect to $V$ consists of the clauses 
\[ \pi :: \effect(RC) \leftarrow \causes(RC) \]
for all clauses $RC \in \Pi$ activated by $V$. 
\end{definition}

Note that in such a relative program, all atoms occurring in a clause are internal. 
We now restrict ourselves to \emph{acyclic} relative programs. 

\begin{definition}
    Let $\Pi$ be a propositional ProbLog program and $V$ an external valuation such that $\Pi^V$ is acyclic. 
    Then the \emph{dependency graph} $\mathrm{Graph}_V(\Pi)$ of $\Pi^V$ is the directed acyclic graph whose node set is the internal vocabulary $I$ and where two nodes are connected by an edge $p \longrightarrow q$ whenever there is a clause $RC \in \Pi^V$ with head $q$ such that either $p$ or $\neg p$ occurs among the causes of $RC$.  
\end{definition}

Now we can establish the relationship to Bayesian networks to define the probabilistic semantics of propositional ProbLog programs.
From now on, we frequently identify a propositional ProbLog program with the probability distribution it induces, for instance allowing us to write ``$\Pi^V$ is Markov to a graph $G$'' when its induced probability distribution is Markov to $G$.  

\begin{definition}
    Let $\Pi$ be a propositional ProbLog program and $V$ an external valuation. Then the semantics of $\Pi$ is given as the Bayesian network on $\mathrm{Graph}_V(\Pi)$ where the conditional probability of an internal atom $p$ given a valuation $v$ on its parents is given by 
    \[\textrm{noisy-or}\left( \left[ \pi_{RC}\in \Pi^V \mid v \models \causes(RC)\right] \right)\]  
    where the noisy-or function associates with every multi-set $S$ of values in the unit interval the number $1 - \prod_{s \in S}(1 - s)$, which is precisely the probability that if each $s \in S$ is the probability of an independent Bernoulli trial, at least one of them would succeed.  
\end{definition}

We illustrate these concepts with a brief example:

\begin{example}
    Consider a situation where something might burn if flammable, and then might smoke. It is particularly likely to smoke if it is not dry. Then $I := \{\mathrm{burns}, \mathrm{smokes}\}$, and $E := \mathrm{flammable}, \mathrm{dry}$. 
    \begin{align}
      \pi_1 &:: \mathrm{burns} \leftarrow \mathrm{flammable}.\\
      \pi_2 &:: \mathrm{smokes} \leftarrow \mathrm{burns}.
      \pi_3 &:: \mathrm{smokes} \leftarrow \mathrm{burns}, \label{line:edge}\neg\mathrm{dry}.
    \end{align}
    An external valuation $v$ determines whether $\mathrm{flammable}$ and/or $\mathrm{dry}$ are true. If $v = \{\mathrm{flammable}\}$, then all three clauses are activated, and the relative program is  
    \begin{align}
      \pi_1 &:: \mathrm{burns}.\\
      \pi_2 &:: \mathrm{smokes} \leftarrow \mathrm{burns}.
      \pi_3 &:: \mathrm{smokes} \leftarrow \mathrm{burns}.
    \end{align}
    This gives rise to a Bayesian network on the graph 
    \[ \mathrm{burns} \rightarrow \mathrm{smokes} \]
    which captures the probability distribution encoded by the probabilistic logic program.
    In particular, the probability of $\mathrm{smokes}$ conditioned on $\mathrm{burns}$ is given by $1 - (1 - \pi_1)(1 - \pi_2)$,
    as the noisy-or function models two separate, independent processes leading to smoke
    (the one that is always activated and the additiona process only activated for non-dry materials). 
    \label{exam:log}
\end{example}

Interventions on propositional programs have been widely studied \cite{cplogic,ruckschloss2021exploiting} as they have a natural realisation through deleting clauses and adding facts. 

\begin{definition}
    Let $L \in \{p, \neg p\}$ be an internal literal and $\Pi$ be a propositional ProbLog program. 
    Then, the \emph{intervention} $\Pi_{\mathrm{do}(L)}$ is obtained from $\Pi$ in two steps. First, remove all clauses with $p$ in the head. Then, if $L$ is a positive intervention, add the fact $p$ to the program. 
\end{definition}
Note that it is irrelevant in which order relativisation to an external valuation and intervention are performed. We can therefore write $\Pi_{\mathrm{do}(L)}^V$ for the program that results from both operations. 
This operation is known to be consistent with the notion of an intervention in the underlying Bayesian network.
For instance, in the situation of Example \ref{exam:log}, releasing smoke unrelated to the fire would remove  Line (\ref{line:edge}) from the program and thus also the arrow from the resulting Bayesian network. This main result of \cite{ruckschloss2021exploiting} can be formalised as follows:  

\begin{proposition}
    Let $\Pi$ be a propositional ProbLog program. Then the probability distribution of $\Pi_{\mathrm{do}(L)}^V$ coincides with the outcome of intervening on $L$ in the Bayesian network of $\Pi^V$, where an intervention on an atom $A$ is modelled as an intervention on $A = \top$, and an intervention on $\neg A$ is modelled as an intervention on $A = \bot$.   
\end{proposition}

This leads us to achieve our first goal: verifying interventions on propositional programs from purely probabilistic information.

\begin{proposition}
    Let $\Pi$ be a propositional probabilistic logic program and let $V$ be an external valuation. 
    Assume that $\Pi^V$ is acyclic faithfully Markov to $\mathrm{Graph}_V(\Pi)$, which furthermore is orientable. 
    Let $\tilde{\Pi}$ be another propositional probabilistic logic program $\tilde{\Pi}$ such that  $\tilde{\Pi}^V$ is acyclic and faithfully Markov to $\mathrm{Graph}_V(\tilde\Pi)$, which encodes the same probability distribution as $\Pi^V$. 
    Then $\mathrm{Graph}_V(\Pi) =\mathrm{Graph}_V(\tilde\Pi)$, and 
    for every internal literal $L$, $\Pi_{\mathrm{do}(L)}^V$ and $\tilde{\Pi}_{\mathrm{do}(L)}^V$ encode the same probability distribution. 
\end{proposition}

\begin{proof}
    If $\Pi^V$ is faithfully Markov to $\mathrm{Graph}_V(\Pi)$, then any other propositional probabilistic logic program $\tilde{\Pi}$ whose relativisation $\tilde{\Pi}^V$ has the same distribution and is faithfully Markov to its dependency graph must have a relative dependency graph that is Markov-equivalent to $\mathrm{Graph}_V(\Pi)$. Since  $\mathrm{Graph}_V(\Pi)$ is orientable, in fact $\mathrm{Graph}_V(\tilde\Pi)= \mathrm{Graph}_V(\Pi)$. Thus, both define the same Bayesian network. 
Since therefore interventions in both programs are equivalent to the same intervention in the same Bayesian network, they furthermore lead to exactly the same post-intervention distribution.    
\end{proof}

Note that we have to require faithfulness to close the gap between orientability and the equivalence of the underlying graphs. 
As suggested by Remark \ref{rem:faithfulness}, deterministic dependencies are highly problematic for faithfulness. 
This means that for the application of these techniques, deterministic logical dependencies should be restricted to the underlying database, which can be of course be a full deductive database where desired. 

\section{Relational Probabilistic Logic Programming}
\label{FOPLP}

In real applications, probabilistic logic programs are often chosen for their ability to go beyond propositional logic to model relational dependencies. 
In this case, even after grounding to a specific database instance, we can take into account our knowledge of the  internal structure of atoms (for instance, different atoms may share a single functor) to postulate further symmetries that can assist in orienting edges of our dependency graph.  

We first introduce relational probabilistic logic programs and their groundings.
Let us fix a relational alphabet $\mathfrak{P}$ in two parts with an \textbf{external vocabulary} $\mathfrak{E} \subset \mathfrak{P}$ and an \textbf{internal vocabulary} $\mathfrak{I} := \mathfrak{P} \setminus \mathfrak{E}$. Here, $\mathfrak{P}$ is a finite relational vocabulary, i.e.~it consists of a finite set of relation symbols, a finite set of constants as well as a countably infinite set of variables. Furthermore,
$\mathfrak{E}$ is a subvocabulary of  $\mathfrak{P}$ containing all the variables and constants of $\mathfrak{P}$ as well as a (possibly empty) subset of the relation symbols of $\mathfrak{P}$. 

As usual in relational logic, an \emph{atom} is an expression of the form $R(t_1, \dots , t_n)$ or $t_1 \doteq t_2$, where $R$ is a relation symbol of arity $n$ and $t_1$ to $t_n$ are constants or variables, and a \emph{literal} is an expression of the form $A$ or $\neg A$ for an atom $A$. 
It is called an \emph{external atom} or \emph{literal} if $R$ is in $\mathfrak{E}$ and an \textbf{internal atom} or \textbf{literal} if $R$ is in~$\mathfrak{I}$. 
A literal $L$ is said to be \textbf{ground} if no variable occurs in it. A \emph{substitution} $\iota$ to a vocabulary $\mathfrak{L}$ is a map from the set of variables to the constants of $\mathfrak{L}$ and, a \emph{grounding} of an expression $\varphi$ to $\mathfrak{L}$ is the result of applying a substitution $\iota$ to $\mathfrak{L}$ to every variable occurrence in $\varphi$.

We proceed to lift the notions of the preceding section to the relational setting.
\begin{definition}[ProbLog clause]
A (relational) \emph{ProbLog clause $RC$} is an expression of the form 
$\left(\pi :: R \leftarrow R_1,...,R_m, L_1,...,L_n. \right)$, written as

$$
\left(
\pi_{RC} :: \effect(RC) \leftarrow \causes(RC) \cup \condition (RC)
\right),
$$
which is given by the following data:
\begin{enumerate}
\item[i)]
an internal atom $R := \effect (RC)$, called the \emph{effect} of $RC$
\item[ii)]
a finite and possibly empty set of internal literals $\causes (RC) := \{ R_1,...,R_m \}$, called the \emph{causes} of~$RC$
\item[iii)]
a finite and possibly empty set of external literals $\condition (RC) := \{ L_1 , ... , L_n \}$, called the \emph{condition} of~$RC$  
\item[iv)]
a \emph{probability} $\pi (RC) \in (0,1)$. 
\end{enumerate}
\label{definition - random clause}
\end{definition}

In the semantics, valuations are replaced by \emph{structures}, which we assume to follow the unique names assumption:

An \textbf{$\mathfrak{I}$-structure} $\Lambda$ consists of a domain $\Delta$, an element of $\Delta$ for every constant in $\mathfrak{I}$, \emph{such that two different constants are mapped to different elements}, and an $n$-ary relation on $\Delta$ for every relation symbol of arity $n$ in~$\mathfrak{L}$.

Whether a logical formula is \textbf{satisfied} by a given $\mathfrak{I}$-structure (under a given \textbf{interpretation} of its free variables) is determined by the usual rules of first-order logic. The semantics of external expressions is defined analogously.

\begin{definition}[Relational ProbLog program]
A \emph{relational Problog program} $\Pi$ is a finite set of (relational) ProbLog clauses.
      
\label{definition - FOPLP}
\end{definition}

A relational ProbLog program induces a probability distribution only when grounded to a database. 

\begin{definition}[Ground variable and external database]
Let $\Pi$ be a  relational ProbLog program and let $\mathcal{E}$ be an $\mathfrak{E}$-structure (which we will also call an \emph{external database}). Let $\mathfrak{E}^{*}$ be the extension of the external vocabulary $\mathfrak{E}$ by constants for every element of $\mathcal{E}$.  
Further, denote by $\mathfrak{P}^*$ and $\mathfrak{I}^*$  the extensions of $\mathfrak{P}$ and $\mathfrak{I}$ respectively by the new constants in $\mathfrak{E}^*$. 
Finally, we write $\mathcal{G} (\mathcal{E})$ for the set of all ground atoms of $\mathfrak{I}^*$. 
\label{def:external database}
\end{definition}

An external database can be used to ground a relational ProbLog program to a propositional program.   

\begin{definition}[Grounding]
Let $\Pi$ be a relational ProbLog program and $\mathcal{E}$ be an external database with vocabulary $E^{*}$ of constants. The \emph{grounding} $\Pi({\mathcal{E}})$ of the program $\Pi$ with respect to $\mathcal{E}$ is the propositional ProbLog program given by the set of the groundings to $E^{*}$ of all random clauses in $\Pi$.  Its associated propositional vocabularies $P$, $I$ and $E$ are given by the sets of all ground atoms of $\mathfrak{P}^*$, $\mathfrak{I}^*$ and $\mathfrak{E}^*$ respectively. 
\label{Semantics of Programs}
\end{definition}

Note that the external database $\mathcal{E}$ can equivalently be viewed as an external valuation to the grounding $\Pi({\mathcal{E}})$, giving rise to the relative propositional program $\Pi^{\mathcal{E}} :=\Pi({\mathcal{E}})^{\mathcal{E}}$ and the \emph{ground graph} $\Graph_\mathcal{E}(\Pi) := \Graph_\mathcal{E}(\Pi(\mathcal{E}))$.   
We can therefore associate every grounding with the Bayesian network corresponding to $\Pi^{\mathcal{E}}$ with its corresponding dependency graph and Bayesian network, where we say that an edge $G_1 \longrightarrow G_2$ is \textbf{induced} by a random clause $RC \in \Pi$ if $G_1$ and $G_2$ are groundings of cause and head atoms in a clause of $\Pi({\mathcal{E}})$ activated by $\mathcal{E}$. 
We close with a brief example illustrating those concepts: 

\begin{example}\label{Eintgr}
    Consider a model of students' likelihood of passing courses as a result of their intelligence:
    \begin{align*}
        \pi:: \mathrm{passes}(X, Y) \leftarrow \mathrm{int}(X), \mathrm{takes}(X,Y).
    \end{align*}
    where $\mathrm{takes}$ is external. 
    A database $\mathcal{E}$ consists of a domain of students and courses, for example $\{\mathit{moe}, \mathit{ana}\}$ for students and $\{\mathit{math}, \mathit{english}, \mathit{sport}\}$ for subjects, and an interpretation of the external predicate ``$\mathrm{takes}$'', for example the pairs \[\{(\mathit{moe},\mathit{english}), (\mathit{moe},\mathit{math}), (\mathit{ana},\mathit{english}), (\mathit{ana},\mathit{sport})\}.\]
    Grounding this one-clause program with respect to $\mathcal{E}$ produces a set of ground rules, each corresponding to a particular student-subject pair.
    The dependency graph $\Graph_\mathcal{E}(\Pi)$  contains a set of edges indicating that a student's intelligence influences their passing likelihood: 

    
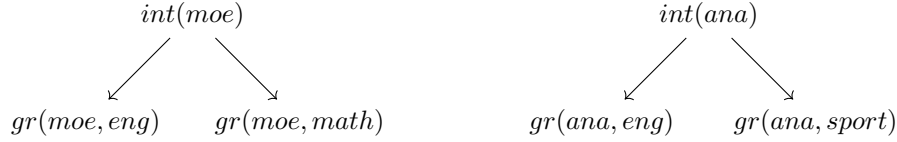
\begin{center}    
    \begin{tikzpicture}[node distance={20mm}] 
        \node (1) {$int(moe)$};
        \node (2) [below right of=1] {$gr(moe, math)$}; 
        \node (3) [below left of=1] {$gr(moe, eng)$}; 
        
        \draw [->] (1) -- (2);  
        \draw [->] (1) -- (3);
        \node (7) [right of=2] {}; 
        \node (4) [right of=7] {$gr(ana, eng)$}; 
        \node (5) [above right of=4] {$int(ana)$}; 
        \node (6) [below right of=5] {$gr(ana, sport)$};
        \draw [->] (5) -- (4);  
        \draw [->] (5) -- (6);
    \end{tikzpicture} 
    \end{center}
\end{example}

\section{Exploiting causal symmetries}

As the edge orientation rules were developed for working with propositional data, they do not take the additional structure of relational specifications into account. 
Thus, we expand the edge orientation rules to reflect the inherent symmetries that may be encoded by a relational specification.

In relational probabilistic logic programs, different edges in a ground graph may stem from the same underlying cause-effect mechanism.
For instance, in Example \ref{Eintgr}, all relations between the predicates $int/1$ and $gr/2$ are oriented in the same direction, from the cause (intelligence) to the effect (grades).

A program structure $S$ captures the causal relationship between two random predicates $q/n$ and $r/m$ as an abstract rule $c$ with $head(c) := q(X_1, ..., X_n)$ and $r(Y_1, ..., Y_m) \in body(c)$. 
Every grounding of $c$ generates a corresponding rule in the grounded program, and all edges will induced by those rules will share the same direction. 
We call this simultaneous behaviour of relations in a ground graph \textbf{causal symmetry}.

\begin{definition}[Causal symmetries]\label{Dsymm}
    Let $G$ be a directed acyclic graph. Let $M$ be a set of sets of directed edges extending adjacencies in $G$ such that for all $m \in M$, also $\{(X,Y) \mid (Y,X) \in m\} \in M$.
    Then $G$ \emph{respects the causal symmetries in $M$} if for all $m \in M$, either for every $(X, Y) \in m$, $(X,Y)$ is also an edge in $G$, or for every $(X, Y) \in m$, $(Y, X)$ is an edge in $G$. 
    If $G$ respects the causal symmetries in $M$, an edge in $G$ is called $M$-orientable if it is present in every $G'$ Markov-equivalent to $G$ which  also respects the causal symmetries in $M$. 
\end{definition}

This definition is illustrated in Example \ref{exam:uwcse} below. 
Clearly, the orientation rules of Proposition \ref{Prop:Meekrules} are still valid for $M$-orientability. The definition above immediately justifies an additional orientation rule: 

\begin{proposition}\label{prop:symtriv}
    Let $G$ be a directed acyclic graph that respects the causal symmetries in $M$, let $(A,B)$ be an $M$-orientable edge in $G$ and let $(C,D)$ be an edge in $G$ such that there is an $m \in M$ with $ (A,B), (C,D) \in m$. Then $(C,D)$ is $M$-orientable.  
\end{proposition}

In his work on relational causal discovery, Maier \cite{maier2014causal} first noticed that symmetries can also be used to orient unshielded symmetric forks, while we were previously restricted to orienting unshielded colliders. 
In the context of general sets of symmetries, we obtain the following result:

\begin{proposition}\label{propsymm}
    Let $G$ be a ground graph with a set of causal symmetries $M$. If $A$ and $C$ are not adjacent in $G$ and $(B,A)$ and $(B,C)$ are edges in $G$, and $(A, B), (C, B) \in m$  for an $m \in M$, then $(B,A)$ and $(B,C)$ are orientable.
\end{proposition}

\begin{proof}
    Assume not. Let $G'$ be Markov-equivalent to $G$ and respecting $M$, such that without loss of generality $(B,A) \notin G'$. As $G$ and $G'$ have the same adjacencies, $(A,B) \in G'$. Since $G'$ respects the causal symmetries in $M$, this implies that $(C,B) \in G'$. 
    By Proposition \ref{Prop:Meekrules}, this means that  $(A,B) \in G'$ and $(C,B) \in G'$ are orientable, and thus they are also in the Markov-equivalent graph $G$, \emph{contrary} to $(B,A) \in G$.      
\end{proof}

We now turn to possible sources of causal symmetries.
Consider the grounded program structure $\Pi^{\mathcal{E}}$, which gives rise to a Bayesian network $B$. 
It would be natural to posit that every clause in the relational program encodes one causal mechanism, 
and thus that all groundings of the same clause should be considered symmetric. 
However, this would assume a priori that the given program encodes correct causal mechanisms, which is precisely what we are trying to verify.
On the other hand, our underlying relational vocabulary can be seen as part of the problem definition, and should be adopted by any candidate program.
This leads us to the broad assumption that the same general process (or at least causal direction) is behind any ground edge involving the same two predicates.  
 
\begin{definition}[Predicate symmetry]  
Let $G := \Graph_V(\Pi)$ be a ground graph of a relational ProbLog program $\Pi$ with internal vocabulary $\mathfrak{I}$. Then the \emph{predicate symmetries} $M_\mathrm{Pred}(G)$ is the set of all sets of the form 
\[\{(P(\vec{a}),R(\vec{b})) \mid P(\vec{a})\textrm{ is adjacent to }R(\vec{b})\textrm{ in }G  \},\] 
where $(P,R)$ ranges over all pairs of distinct predicates in $\mathfrak{I}$.    
\end{definition}

Predicate symmetry encodes the assumption that the predicates of relations are the only determinants of their cause-effect direction.

For instance, in Example \ref{Eintgr}, it means that the direction of causal flow between grades and intelligence is the same in every student, and for every subject.
This often creates symmetric forks, which can occur whenever a cause is missing a variable present in the head.

For instance, Proposition \ref{propsymm} renders all the forks in Example \ref{Eintgr} orientable. If there were also students in the database that only take a single course, their effects could then also be oriented by appeal to Proposition \ref{prop:symtriv}. 
Thus, the two new orientation rules are particularly powerful when applied together. 

In terms of the program structure, the assumption of predicate symmetry is equivalent to the assertion that the predicate dependency graph of any program structure does not have 2-cycles.

Predicate symmetry is a powerful tool for orienting (instances of) relational probabilistic programs, but it is also a strong condition on the structure of the underlying causal mechanisms. 
Most clearly, this assumption is inappropriate when our initial program structure itself has 2-cycles in its predicate dependency graph, as can be typically encountered in time-stratified programs: 
\begin{example}
    Consider a program that models a disease against which immunity is acquired:
    \begin{align*}
        \pi_1&::\mathrm{ill}(x,t) \leftarrow \neg \mathrm{resistant}(x,t).\\
        \pi_2&::\mathrm{resistant}(x,t) \leftarrow \mathrm{ill}(x,t'), \mathrm{time\_step}(t,t').
    \end{align*}
    In the instantiating database, the interpretation of $\mathrm{time\_step}$ guards against any induced cycle. 
    However, causal flow is clearly reversed in the two encoded processes, and therefore predicate symmetry is unsuitable. 
\end{example}

In such cases, rather than defining symmetries, it is more appropriate to prescribe the orientations of edges between time-steps, and only consider those graphs which match that orientation. 
Indeed, this is exactly the background knowledge discussed by Meek \cite{CausalInferenceAndCausalExplanationWithBackgroundKnowledge}, who showed that the rules of Proposition \ref{Prop:Meekrules} remain complete when augmented by such prescribed edge directions. 
The relational structure facilitates such background knowledge: in our implementation, a predicate $\textrm{direction/2}$ can be specified as part of the database, orienting all edges $(R(a_1, \dots, a_n), S(b_1, \dots, b_n))$ such that $\mathrm{direction}(a_i,b_j)$ is true for any $i$ and $j$.  

\begin{example}
    
We now illustrate the usefulness of exploiting symmetric behaviour for the orientation process on a program structure taken from the cplint example suite \cite{cplint} and  derived from the UWCSE dataset of university webpages \cite{uwcse}.

\begin{align*}
\pi_1 &:: \mathrm{advisedby}(A, B) \leftarrow 
    \mathrm{r11}(A, B, C), \mathrm{student}(A),\ \mathrm{professor}(B),\ \mathrm{project}(C, A), \nonumber \\
    &\phantom{:: \mathrm{advisedby}(A, B) \leftarrow\ }
    \mathrm{project}(C, B). \\
\pi_2 &:: \mathrm{advisedby}(A, B) \leftarrow 
    \mathrm{student}(A),\ \mathrm{professor}(B),\ 
    \mathrm{ta}(C, A),\ \mathrm{taughtby}(C, B). \\
\pi_3 &:: \mathrm{r11}(A, B, C) \leftarrow 
    \mathrm{publication}(D, A, C),\ \mathrm{publication}(D, B, C).
\end{align*}

A small fragment of a ground graph, instantiated with a typical database for this domain, can be seen in Figure \ref{uwcse1ggraph}.

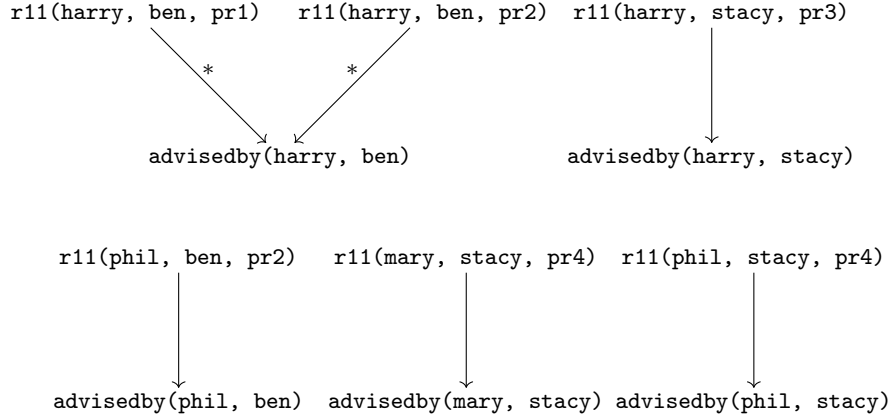
\begin{figure}
    \centering
    \begin{tikzpicture}[node distance={19mm},   every node/.style={inner sep=1pt}]
        {\footnotesize 
        \node (n1) {\texttt{advisedby(harry, ben)}};
        \node (n0) [above of=n1]{};
        \node (n2) [left of=n0] {\texttt{r11(harry, ben, pr1)}}; 
        \node (n3) [right of=n0] {\texttt{r11(harry, ben, pr2)}};
        \draw [->]  (n2) -- node[midway, above]{*}(n1);  
        \draw [->]  (n3) -- node[midway, above]{*}(n1);  
        \node(n10) [below left of=n1] {\texttt{r11(phil, ben, pr2)}};
        \node (n11) [below of=n10] {\texttt{advisedby(phil, ben)}};
        \draw [->] (n10) -- (n11);

        \node(n10a) [right of=n10] {};
        \node (n6) [right of=n10a] {\texttt{r11(mary, stacy, pr4)}};
        \node (n7) [below of=n6] {\texttt{advisedby(mary, stacy)}};
        \draw [->] (n6) -- (n7);

        \node (n6a1) [right of=n6] {};
        \node (n8) [right of=n6a1] {\texttt{r11(phil, stacy, pr4)}};
        \node (n9) [below of=n8] {\texttt{advisedby(phil, stacy)}};
        \draw [->] (n8) -- (n9);

        \node (n6a2) [above right of=n6] {};
        \node (n5) [right of=n6a2] {\texttt{advisedby(harry, stacy)}};
        \node (n4) [above of=n5] {\texttt{r11(harry, stacy, pr3)}};
        \draw [->] (n4) -- (n5);
        }
    \end{tikzpicture} 
    \caption{(Fragment of) a ground graph arising from the UWCSE program structure}
    \label{uwcse1ggraph}
\end{figure}

As there is only one unshielded collider, an algorithm taking into account only the ground graph could orient only the starred arrows. 
However, if we assume predicate symmetry, then the set of all edges from a relation with r11/3 to a relation with advisedby/2 forms a causal symmetry. 
Thus, once the starred arrows are oriented, all the remaining arrows can also be oriented likewise according to Proposition \ref{prop:symtriv}. 
Therefore, if we accept the assumption of predicate symmetry, we can use this induced ProbLog program for interventional reasoning. 
\label{exam:uwcse}
\end{example}

This example illustrates the general point that symmetries are very powerful in exploiting the local structure of forks or colliders in one part of the ground dependency graph to orient those edges where due to limitations on the size or structure of the database the same edge appears as a 1:1 relation. 

An implementation of the ideas presented here (in Prolog, using Logtalk \cite{logtalk} and tabling) is available on GitHub as part of PLP-BN tools [\url{https://github.com/weitkaemper/plpbn-tools}], along with sample files to load Example \ref{exam:uwcse}. 

 \section{Related and further work}
 To the best of our knowledge, this is the first contribution to study the problem of identifying causal effects specifically for probabilistic logic programs. 
 In a context of lifted Bayesian networks, Maier \cite{maier2014causal} already identified causal symmetries as a powerful factor in causal inference. 
 Unlike our setting, Maier operates on a so-called abstract ground graph derived directly from the relational specification, which has some consequences: 
 Firstly, Maier assumes this graph to be acyclic, which seems an even stronger condition than the assumption of predicate symmetry. 
 Secondly, his framework seems to imply that the probability distribution is known independently of any background logical theory or ground database, which curtails deterministic specification even further and implies a domain-agnostic learning process. 
 In contrast, probabilistic inductive logic programming typically learns from a single ground database (``mega-example''), and therefore causal analysis should also be domain-specific. 
 
 We highlight two  promising directions for future work. 
 Firstly, deterministic relationships are arguably more prevalent in probabilistic logic programming than in other statistical relational approaches such as relational Bayesian networks or Markov logic networks. 
 Therefore, the breach of faithfulness caused by deterministic dependencies can only partially be mitigated by pushing the logical conditions into the deductive database used in grounding. 
 This suggests that adapting to  recent score-based causal inference approaches such as determinism-aware greedy equivalent search \cite{DGES}, which are designed for partially deterministic contexts, may be particularly relevant to probabilistic logic programming.

Secondly, while our symmetry-aware orientation rules are correct, they are not complete. Given the broader relevance of symmetries beyond probabilistic logic programming, it would  be worthwhile to study this problem more generally, for instance by establishing the complexity class of edge orientation under symmetry constraints.

\section{Conclusion}
We established a general method to verify the causal content implied by a probabilistic logic program. 
In the case of a propositional probabilistic logic program, this was achieved by mapping the program to a Bayesian network and applying the known orientation rules of Meek \cite{CausalInferenceAndCausalExplanationWithBackgroundKnowledge}. 
We then extended our approach to take into account symmetries arising from grounding a relational probabilistic logic program.
In addition to a general, flexible method for handling symmetries suggested by the underlying domain structure, we introduced predicate symmetry to encode the assumption that he predicate symbols are the only determinants of cause-effect direction. 
Such symmetries can be exploited to orient more ground structures than would otherwise have been possible.
In particular, symmetric forks can now be oriented as well as colliders,
and orientations can be propagated along sparse relational structures.   
In our portable Prolog-based implementation in PLP-BN tools [\url{https://github.com/weitkaemper/plpbn-tools}], we provide access to both predicate symmetries and sets of symmetries prescribed by the user, exposed as a transparent Logtalk API.

\paragraph{\textbf{Acknowledgements}} This publication was supported by LMUexcellent, funded by the Federal Ministry of Education and Research (BMBF) and the Free State of Bavaria under the Excellence Strategy of the Federal Government and the Länder. The authors also wish to thank the anonymous reviewers on the programme committee of IJCLR 2025, whose feedback significantly improved the clarity of the presentation. 

\bibliographystyle{splncs04}
\bibliography{bib.bib}

\begin{thebibliography}{10}
\providecommand{\url}[1]{\texttt{#1}}
\providecommand{\urlprefix}{URL }
\providecommand{\doi}[1]{https://doi.org/#1}

\bibitem{slipcover}
Bellodi, E., Riguzzi, F.: Structure learning of probabilistic logic programs by
  searching the clause space. Theory and Practice of Logic Programming
  \textbf{15}(2),  169--212 (Jan 2014). \doi{10.1017/s1471068413000689}

\bibitem{ProbLog}
De~Raedt, L., Kimmig, A., Toivonen, H.: {ProbLog}: A probabilistic {P}rolog and
  its application in link discovery. In: Proceedings of the 20th International
  Joint Conference on Artificial Intelligence (IJCAI 2007). vol.~7, pp.
  2462--2467. AAAI Press (01 2007)

\bibitem{liftcover}
Fadja, A.N., Riguzzi, F.: Lifted discriminative learning of probabilistic logic
  programs. Mach. Learn.  \textbf{108}(7),  1111--1135 (2019).
  \doi{10.1007/S10994-018-5750-0}

\bibitem{DGES}
Li, L., Dai, H., Ghothani, H.A., Huang, B., Zhang, J., Harel, S., Bentwich, I.,
  Chen, G., Zhang, K.: On causal discovery in the presence of deterministic
  relations. In: Advances in Neural Information Processing Systems 38: Annual
  Conference on Neural Information Processing Systems 2024, NeurIPS 2024,
  Vancouver, BC, Canada, December 10 - 15, 2024 (2024),
  \url{http://papers.nips.cc/paper\_files/paper/2024/hash/ec52572b9e16b91edff5dc70e2642240-Abstract-Conference.html}

\bibitem{maier2014causal}
Maier, M.: Causal discovery for relational domains: Representation, reasoning,
  and learning. Ph.D. thesis, University of Massachusetts at Amherst (2014)

\bibitem{CausalInferenceAndCausalExplanationWithBackgroundKnowledge}
Meek, C.: {C}ausal {I}nference and {C}ausal {E}xplanation with {B}ackground
  {K}nowledge. In: Proceedings of the Eleventh Conference on Uncertainty in
  Artificial Intelligence. pp. 403--410. UAI'95, Morgan Kaufmann Publishers
  Inc., San Francisco, CA, USA (1995).
  \doi{https://dl.acm.org/doi/10.5555/2074158.2074204}

\bibitem{FaithfulnessHoldsAlmostAlways}
Meek, C.: Strong completeness and faithfulness in bayesian networks. In:
  Proceedings of the Eleventh Conference on Uncertainty in Artificial
  Intelligence. pp. 411--418. UAI'95, Morgan Kaufmann Publishers Inc., San
  Francisco, CA, USA (1995)

\bibitem{logtalk}
Moura, P.: Logtalk. Ph.D. thesis, Universidade da Beira Interior (2003)

\bibitem{Causality}
Pearl, J.: {C}ausality. Cambridge University Press, Cambridge, UK, 2nd edn.
  (2000). \doi{https://doi.org/10.1017/CBO9780511803161}

\bibitem{uwcse}
Richardson, M., Domingos, P.M.: Markov logic networks. Mach. Learn.
  \textbf{62}(1-2),  107--136 (2006). \doi{10.1007/S10994-006-5833-1}

\bibitem{riguzzi2022foundations}
Riguzzi, F.: Foundations of probabilistic logic programming: Languages,
  semantics, inference and learning. River Publishers (2022)

\bibitem{cplint}
Riguzzi~et.al., F.: {cplint} (2018), \url{https://github.com/friguzzi/cplint}

\bibitem{ruckschloss2021exploiting}
R{\"{u}}ckschlo{\ss}, K., Weitk{\"{a}}mper, F.: Exploiting the full power of
  pearl's causality in probabilistic logic programming. In: Proceedings of the
  International Conference on Logic Programming 2022 Workshops. {CEUR} Workshop
  Proceedings, vol.~3193. CEUR-WS.org (2022),
  \url{https://ceur-ws.org/Vol-3193/paper1PLP.pdf}

\bibitem{independencies}
R{\"{u}}ckschlo{\ss}, K., Weitk{\"{a}}mper, F.: On the independencies hidden in
  the structure of a probabilistic logic program. In: Proceedings 39th
  International Conference on Logic Programming, {ICLP} 2023. {EPTCS},
  vol.~385, pp. 169--182 (2023). \doi{10.4204/EPTCS.385.17}

\bibitem{CausationPredictionAndSearch}
Spirtes, P., Glymour, C., Scheines, R.: {C}ausation, {P}rediction, and
  {S}earch. MIT press, 2nd edn. (2000),
  \url{https://www.researchgate.net/publication/242448131_Causation_Prediction_and_Search}

\bibitem{pcalgo}
Spirtes, P., Glymour, C.: An algorithm for fast recovery of sparse causal
  graphs. Social science computer review  \textbf{9}(1),  62--72 (1991)

\bibitem{cplogic}
Vennekens, J., Denecker, M., Bruynooghe, M.: {CP}-logic: A language of causal
  probabilistic events and its relation to logic programming. Theory and
  Practice of Logic Programming  \textbf{9}(3),  245--308 (2009)

\end{thebibliography}

\end{document}